\pdfoutput=1
\documentclass{article}

\PassOptionsToPackage{numbers,sort,compress}{natbib}
\usepackage[main, final]{neurips_2026}
\usepackage[utf8]{inputenc} % allow utf-8 input
\usepackage[T1]{fontenc}    % use 8-bit T1 fonts
\usepackage[hidelinks]{hyperref}       % hyperlinks
\usepackage{url}            % simple URL typesetting

\usepackage{booktabs}       % professional-quality tables
\usepackage{amsfonts}       % blackboard math symbols
\usepackage{nicefrac}       % compact symbols for 1/2, etc.
\usepackage{microtype}      % microtypography
\usepackage{xcolor}         % colors

\usepackage{graphicx} 

\usepackage{tcolorbox}
\tcbuselibrary{breakable, skins}
\usepackage{fvextra}

\usepackage{caption}

\DeclareCaptionType{prompt}[Prompt][List of Prompts]

\tcbset{
  promptstyle/.style={
    breakable,
    enforce breakable,
    skin=enhanced,
    colback=gray!5,
    colframe=black!60,
    fonttitle=\bfseries\small,
    left=4pt, right=4pt, top=4pt, bottom=4pt,
    boxrule=0.4pt,
    pad at break=2mm,
  }
}

\usepackage{subcaption}

\usepackage{tikz}
\usetikzlibrary{positioning,arrows.meta}

\usepackage{amsmath,amsthm,amssymb}

\theoremstyle{plain}
\newtheorem{thm}{Theorem}
\newtheorem{prop}[thm]{Proposition}
\newtheorem{lem}[thm]{Lemma}

\theoremstyle{definition}
\newtheorem{defn}[thm]{Definition}

\newcommand{\bb}[1]{\mathbf{#1}}
\newcommand{\B}[1]{\mathbb{#1}}
\newcommand{\C}[1]{\mathcal{#1}}
\newcommand{\D}[1]{\operatorname{\rm #1}}

\newcommand{\defeq}{\stackrel{\hbox{\tiny \rm def}}{=}}
\newcommand{\indep}{\mathrel{\perp\llap{\(\,\,\perp\)}}}

\title{Mitigating Label Bias with\\Interpretable Rubric Embeddings}

\author{%
  Calvin Isley \\
  Harvard Kennedy School\\
  Harvard  University\\
  Cambridge, MA 02139 \\
  \texttt{cisley@g.harvard.edu} \\
  \And 
  Johann D. Gaebler \\
  Stern School of Business\\
  New York University\\
  New York, NY 10012\\
  \texttt{jgaeb@stern.nyu.edu} \\
  \And 
  Sharad Goel \\
  Harvard Kennedy School\\
  Harvard  University\\
  Cambridge, MA 02139 \\
  \texttt{sgoel@hks.harvard.edu} \\
}

\begin{document}

\maketitle

\begin{abstract}
  Statistical decision algorithms are increasingly deployed in domains where ground-truth labels are hard to obtain, such as hiring, university admissions, and content moderation. In these settings, models are typically trained on historical human evaluations---for example, using past hiring decisions as a proxy for true applicant quality. However, if past evaluations unjustly favor certain groups, models trained on these labels may inherit those biases. To address this problem, we propose basing predictions on rubric embeddings, a representation framework that replaces standard black-box embeddings with features derived from expert-defined criteria that align with the underlying construct of interest. By anchoring predictions to semantically meaningful dimensions, this approach guards against biased proxy signals. We provide both theoretical and empirical evidence that rubric embeddings mitigate label bias under plausible conditions. Empirically, we evaluate our method on a novel dataset of applications to a large master’s program. We find that models trained on rubric embeddings reduce group disparities while improving measures of cohort quality. Our results suggest that basing predictions on interpretable, domain-grounded representations offers a practical approach to learning in the presence of biased labels.
\end{abstract}

\section{Introduction}
\label{sec:intro}

In many supervised learning problems, models are trained to predict labels that are treated as ground truth for the quantity of interest. This includes, for example, predicting object categories from images, transcripts from speech, and disease status from medical data. In many domains, however, ground-truth labels do not exist or are prohibitively difficult to obtain. Often the target quantity is a latent construct that is only imperfectly defined and can only be observed for a non-representative subset of individuals. For instance, in hiring, one may wish to select the candidate most likely to succeed in a role, but ``success'' is inherently ambiguous and is typically only observed for candidates who were hired in the past.

In the absence of ground-truth labels, a common empirical strategy is to predict historical human evaluations, such as hiring decisions or evaluator ratings of past applicants. But these proxy labels may be biased, reflecting implicit and unjustified preferences for certain groups. In this paper, we show how models trained on proxy labels can inherit and amplify bias. Using a novel dataset of applications to a master's program, we show that the common approach of training models on black-box text embeddings can reproduce biases present in historical evaluations. In particular, we find that such embeddings encode sensitive attributes like gender, even when explicit markers are removed, and that models trained on these representations are, under certain conditions, guaranteed to replicate biases in the data. We provide both empirical and theoretical evidence for this phenomenon, drawing on the framework of correspondence experiments from the social sciences.

We then analyze three past approaches to mitigate this problem---orthogonalization, redaction, and marginalization---and discuss their limitations. In brief, orthogonalization induces demographic parity, which can inadvertently penalize qualified applicants from certain groups. Redaction, while sound in theory, can fail in practice, as residual signals of gender often remain in the embeddings. Finally, marginalization ensures that certain types of biases are removed, but can introduce other forms.

Given these limitations, we propose instead basing predictions on rubric embeddings, representations constructed from expert-defined criteria aligned with the outcome of interest.
In the case of university admissions, we construct hundreds of rubric items that encode, for example, past coursework and grades, work experience, and letter evaluations. We then use LLMs to score applications along the rubric dimensions to create a semantically grounded representation of the data. We show that basing predictions on rubric embeddings addresses the central limitations of past methods, yielding cohorts of students that are both highly qualified and demographically diverse. Our results suggest that rubric embeddings provide a practical strategy for mitigating label bias in a wide range of decision-making contexts, addressing a long-standing challenge in the equitable design of algorithms~\cite{chohlas2023designing}. 
We conclude by connecting our work to two common understandings of discrimination in the law---disparate treatment and disparate impact---and discussing some of the limitations of our approach. 

Analysis code and model outputs needed to reproduce the reported figures and statistics are available at \url{https://github.com/calisley/label-bias}. 

\paragraph*{Related Work.}

Our work connects three strands of research that have developed largely in parallel: (1) algorithmic fairness and, specifically, the problem of label bias; (2) emerging work on concept bottleneck models and rubric representations, which considers how best to construct and leverage interpretable, semantically meaningful features; and (3) the design of correspondence experiments, also known as audit studies, popular for measuring discrimination in the social sciences.  We discuss the first two areas below, and then consider correspondence experiments in detail in Section~\ref{sec:problem}.

An extensive literature in algorithmic fairness develops techniques for identifying and mitigating biases in algorithmic systems \citep[see, e.g.,][]{hardt2016equality, dwork2012fairness, kleinberg2016inherent, chouldechova2018case, kusner2017counterfactual,corbett-davies_measure_2023,chohlas2023designing,corbett2017algorithmic}. We focus on \emph{label bias}, which arises when models are trained on proxies that systematically diverge from the underlying outcomes of interest. For example, an algorithm that predicts future healthcare expenditures as a proxy for future patient needs can induce racial disparities, since White patients, on average, spend more than equally sick Black patients~\citep{obermeyer2019dissecting}. Ideally, one would acquire more accurate labels, but doing so is often expensive or otherwise infeasible.
To avoid collecting new data, past work has considered reweighting the existing, biased data under structural assumptions about the labeling process~\citep{jiang_identifying_2020}, or adjusting the loss to account for group-dependent label noise~\citep{wang_fair_2021}. These approaches aim to correct bias at the level of the training objective, but rely on strong assumptions about how bias enters the labeling process, which may not hold in practice.
\citet{label-bias} show that, in the presence of label bias, selectively removing features can improve performance on the true, unobserved label. Theirs is an important theoretical insight, but they do not offer a constructive method for selecting the features to remove.
Building on that work, we show that using rubric embeddings provides a practical, constructive method for selecting which features to retain, effectively operationalizing this idea of feature removal; we analytically relate our approach to their observation.

Concept bottleneck models (CBMs) structure predictions by factoring them through a set of interpretable, human-defined features~\citep{koh_concept_2020}. Originally developed for image data, CBMs have been extended to text, where LLMs are used to score documents on predefined or generated concept sets, often achieving performance competitive with black-box approaches~\citep{tan_interpreting_2025, sun_concept_2025,
ludan_interpretable-by-design_2024}. 
Rubric embeddings (also known as rubric-derived representations) are a closely related approach that uses LLMs to extract structured features from text, and have likewise shown competitive or superior performance to
black-box embeddings in low-data regimes~\citep{lin2024spur, balek2024llm,
demirel2026rubric}. Although CBMs and rubric-derived representations have
developed largely in parallel, both share the central idea of substituting a
small, semantically meaningful feature set for a generic embedding. Most directly related to our own work,
\citet{ludan_interpretable-by-design_2024} and \citet{demirel2026rubric} both
use LLMs to generate a collection of interpretable, semantically
relevant features. 
Further, human-graded rubrics have been used to establish ground-truth resume quality in audit studies of LLM evaluators~\citep{xu2025selfpreferencing}.
This literature has focused primarily on interpretability, sample efficiency, and robustness to distribution shift~\citep{yan_robust_2023, choi2024adaptiveconceptbottleneckfoundation}.
However, to our knowledge, the use of concept-based representations or rubrics to mitigate label bias has not been studied.

\section{Label Bias and Black-Box Embeddings: A Case Study in  Admissions}

\label{sec:problem}

We begin by illustrating the problem of label bias and discussing its connection to black-box embeddings. 
To do so, we imagine designing an algorithm to evaluate applicants to a graduate degree program. 
We base our empirical analysis on datasets derived from real applications to a large master's program in public policy, where we synthetically inject increasing degrees of bias into the real, observed admissions decisions. This approach allows us to study label bias across a range of settings.
In the process, we build on and formalize the notion of correspondence studies, a widely used method for measuring discrimination in the social sciences. 

Specifically, we start with 1,112 applications submitted during the 2025--2026 application cycle. Each application contains structured information, such as basic demographics and standardized test scores, along with four kinds of unstructured materials: a resume of professional and volunteer experience; transcripts from all previous undergraduate and graduate institutions; three letters of recommendation; and short essays describing applicants' backgrounds, reasons for applying, and future plans. Admissions officers evaluate each application through a structured process resulting in a 4--20 point score. These scores determine admissions decisions and are distributed approximately normally, with most applicants receiving between 10 and 17 points (Fig.~\ref{fig:score_dist}).

For the purposes of our stylized empirical analysis, we treat these numeric scores as ground-truth labels. 
While these observed scores may themselves contain noise or bias, they serve as a fixed reference point for evaluating the relative performance of different modeling approaches.
We then explicitly introduce label bias by adding a group-specific noise term to the scores.
Specifically, given the actual score \(Y_i\) for the \(i\)-th applicant, we define the proxy score
\begin{align}
\label{eq:biased_outcome}
    Y_i' \defeq \begin{cases}
        Y_i + Z_i & G_i = m, \\
        Y_i - Z_i & G_i = f,
    \end{cases}
    \qquad Z_i \sim \C N(b, 1), \qquad Z_i \indep X_i, G_i,
\end{align}
where \(G_i \in \{f, m\}\) denotes whether the applicant is male or female and \(b\) gives the size of the advantage for male applicants. We vary \(b\) from 0, representing no advantage, to 2.5, representing a large advantage for male applicants. 

Throughout, we fit models on the proxy label \(Y'\), and then consider both the ``true'' quality (as indicated by \(Y\)) and demographic composition of the cohort of top-ranked applicants under the learned model.
To fit these models, we must construct a representation of the unstructured data; one common approach 
is to embed the data using general-purpose black-box text embeddings~\citep{reimers2019sentence, devlin2019bert, neelakantan2022text}.

\subsection{Theoretical analysis}
\label{sec:theory}
 
We quantify the bias of the fitted models by
drawing on the extensive literature on correspondence experiments~\citep{gaebler2025auditing}. Correspondence experiments, also known as ``audit studies,'' measure the extent to which decisions in hiring, admissions, and related settings vary for individuals who are (nearly) identical except for some protected characteristic~\citep{bertrand2017field, gaddis2019understanding}. In practice, researchers experimentally manipulate materials to alter decision-makers' perceptions of  group membership.
In one of the first such studies, \citet{bertrand2004emily} sent out job applications that were identical in every respect but the name of the applicant. Those with names suggesting they were Black (e.g., `Lakisha' or `Jamal') received significantly fewer callbacks than those with names suggesting they were White (e.g., `Emily' or `Greg')---providing evidence of racial bias.

To formalize this approach to measuring bias, let \(T(x, g) : \C X \times \{m, f\} \to \C X\) denote a correspondence experiment \emph{manipulation}: a (possibly stochastic) transformation of an applicant's features \(X \in \C X\) with \(T(x, m)\) and \(T(x, f)\) representing the ``male-presenting'' and ``female-presenting'' versions of \(x \in \C X\), respectively.\footnote{%
    For notational conventions, full details on regularity conditions, and proofs, see Appendix~\ref{app:math}.
}
For example, one such transformation might change the (raw, unstructured) application materials so that the  applicant's listed name and the pronouns their letter writers use to refer to them accord with the target gender. 
Now, given an algorithm \(h(x) : \C X \to \B R\), we define the bias of \(h\) under the correspondence experiment determined by \(T\) as follows.

\begin{defn}[\(\D {bias}_T(h)\)]
\label{defn:bias}
    We define the \emph{bias} with respect to \(T(x, g)\) as
    \begin{equation}
    \label{eq:cor_exp_bias}
        \D {bias}_T(h) \defeq \B E[h(T(X, m)) - h(T(X, f))].
    \end{equation}
\end{defn}

Def.~\ref{defn:bias} captures the average extent to which predictions for an applicant differ between their male- and female-presenting features.
With this framework, we can now explicitly derive the bias of models trained with a broad range of biased labels, including those defined in Eq.~\eqref{eq:biased_outcome}.
(Proofs of Prop.~\ref{prop:bias} and other results are given in Appendix~\ref{app:math}.)

\begin{prop}
\label{prop:bias}
  Let \(Y, B \in \B R\) be random variables, and let \(G \in \{m, f\}\), \(X \in \C X\), and \(T(x, g) : \C X \times \{m, f\} \to \C X\) be as above. Let \(r(x) : \C X \to \C R\) and define \(R \defeq r(X)\). Suppose \(B \indep X \mid G\). Then
  \begin{equation}
  \label{eq:bias_decomp}
    \D {bias}_T(\nu) = \D {bias}_T(\mu) + \left(\B E[B \mid G = m] - \B E[B \mid G = f]\right) \cdot \B E[\pi(T(X, m)) - \pi(T(X, f))]
  \end{equation}
  where
  \(\nu(x) \defeq \B E[Y + B \mid R = r(x)]\), \(\mu(x) \defeq \B E[Y \mid R = r(x)]\), and \(\pi(x) \defeq \Pr(G = m \mid R = r(x))\).
\end{prop}

Prop.~\ref{prop:bias} considers the bias of an algorithm \(\nu(x) = \B E[Y + B \mid R = r(x)]\) that perfectly estimates a biased label \(Y + B\) given the available information in some representation \(R\) of an application, such as a black-box text embedding.
In it, we assume that \(B\) perturbs the ground truth label \(Y\) in a way that only depends on group membership \(G\), as holds for our synthetic datasets described in Eq.~\eqref{eq:biased_outcome}.
In this case, the proposition decomposes \(\D {bias}_T(\nu)\) into three terms: (1) \(\D {bias}_T(\mu)\), the baseline bias in the algorithm \(\mu\) present when predicting the true label \(Y\); (2) \(\B E[B \mid G = m] - \B E[B \mid G = f]\), the difference in the average perturbation \(B\) across groups, a measure of how much a group is advantaged; 
and (3) \(\B E[\pi(T(X, m)) - \pi(T(X, f))]\), the difference in the probability that male- and female-presenting features belong to a male applicant. By Eq.~\eqref{eq:biased_outcome},
\[
    \B E[B \mid G = m] = \B E[Z \mid G = m] = b, \qquad \B E[B \mid G = f] = \B E[-Z \mid G = f] = -b.
\]
Consequently, for our stylized example, 
\begin{equation}
\label{eq:bias_synth_decomp}
    \D {bias}_T(\nu) = \D {bias}_T(\mu) + 2b \cdot \B E[\pi(T(X, m)) - \pi(T(X, f))].
\end{equation}
Critically, \(\D {bias}_T(\nu)\) depends on the extent to which applications can be manipulated to look more male or more female. In other words, to the extent that \(R\) implicitly encodes group membership, it is more likely to inherit any biases present in the labels. 
Our example assumes a particular form of label bias to facilitate analysis, but the core qualitative insight extends beyond this specific formulation, a pattern we explore empirically in the subsequent sections.

Eq.~\eqref{eq:bias_decomp} further indicates that \(\D {bias}_T(\nu)\) depends on the baseline bias, \(\D {bias}_T(\mu) \), of the algorithm when predicting the unbiased, ground-truth label \(Y\). It is perhaps surprising that an algorithm can exhibit such bias even when trained on an unbiased label. Prop.~\ref{prop:baseline_bias} below sheds light on this phenomenon by deriving an alternative expression for the bias of an algorithm.
For clarity and tractability, we state the result for linear predictors but the intuition holds more broadly.

\begin{prop}
\label{prop:baseline_bias}
    Let \(X \in \C X\) and \(R = r(X) \in \B R^p\) be as above, and let \(Y \in \B R\) be 
    square-integrable. Let \(h_R(r) = \alpha_0 + r^\top \alpha\) be the population 
    OLS predictor of \(Y\) from \(R\). Define the audit-induced representation shift
    \[
        s \defeq \B E\!\left[r(T(X,m)) - r(T(X,f))\right].
    \]
    Then
    \begin{equation}
    \label{eq:baseline_bias}
        \D{bias}_T(h_R) = s^\top \D{Var}(R)^{-1} \D{Cov}(R, Y).
    \end{equation}
\end{prop}

Prop.~\ref{prop:baseline_bias} illustrates that when features shift under the gender manipulation \(T\), bias arises to the extent that those features are predictive of \(Y\). Writing the population coefficient vector as \(\alpha = \D{Var}(R)^{-1} \D{Cov}(R, Y)\), Eq.~\eqref{eq:baseline_bias} becomes \(\D {bias}_T(h_R) = s^\top \alpha\), meaning that bias is nonzero precisely when the average representation shift \(s\) has a component along the predictive direction \(\alpha\). As a result, omitted variable bias can lead to audit study bias, even in the presence of unbiased labels. As a simple illustration, note that when \(R\) is whitened (i.e., when \(\D {Var}(R) = I_p\)), Eq.~\eqref{eq:baseline_bias} further simplifies to \(s^\top \D{Cov} (R, Y)\), meaning that each feature individually contributes to the bias to the extent that it correlates with the outcome and changes under the transformation. If some feature relevant to \(Y\) is correlated with gender (for instance, if female applicants are more likely to engage in public service) but is not captured by \(R\), then the components of \(R\) that \emph{are} captured and that shift with gender act as proxies for the omitted feature, inducing bias. 

The distinction between omitted variable bias and label bias is subtle, and they often lead to indistinguishable issues in practice. Consider the widely reported case of an internal hiring algorithm at Amazon that penalized resumes containing terms like ``women's'' (e.g., as in ``women's chess club captain'')~\citep{dastin_insight_2018}. Conventional wisdom suggests that the penalty reflected biases in the historical hiring data. However, Proposition~\ref{prop:baseline_bias} suggests omitted variable bias as an alternative culprit. A model trained using data representations that omit job-relevant features correlated with gender would likely exhibit similar biases---a point we return to below.

\subsection{Empirical analysis}
\label{sec:empirical}
We next empirically examine how much gendered information is implicitly encoded in standard text embeddings, and analyze its consequences for cohort selection. 
To convert the raw application materials to vector representations, we first extract their contents with Microsoft Document Intelligence~\citep{microsoft2024documentintelligence}. We then separately embed each of the four document types (resumes, essays, letters, and transcripts) using OpenAI's \texttt{text-embedding-3-large} model~\citep{openai2023textembedding}.
Given the limited number of applications,
we use only the first 250 dimensions of the embeddings for each category~\citep{kusupati2022matryoshka}. 
We additionally include some structured data, but we omit gender and other explicit demographic features because deployed admissions models likely cannot legally include them~\citep{2023students}.
This process results in an approximately 1,000-dimensional embedding.

We next use these vector representations to predict the biased labels \(Y'\), for \(b \in [0, 2.5]\).
We specifically fit ridge regression models, which we denote by \(h\), using the \texttt{glmnet} package~\citep{glmnet_package}. To generate out-of-sample predictions, we use \(10 \times 10\)-fold nested cross-validation: For each of the 10 outer folds, we fit and tune a ridge model on the remaining nine folds using 10-fold cross-validation before predicting on the held-out fold. We then calibrate the resulting models using linear scaling on the training folds. We repeat this process 20 times, combining standard errors across iterations using Rubin's rules~\citep{rubin1987multiple}. We fit models on a 4-core CPU, running for roughly 12 hours.

Following~\citet{gaebler2025auditing}, we define \(T(x, g)\) to account for many signals of gender beyond names that rich inputs like our application materials contain.
We use language models to manipulate gendered cues about applicants, including: (1) third-person pronouns; (2) gendered titles (e.g., ``Mr.'' or ``Mrs.''); (3) kinship terms (e.g., ``son'' or ``daughter''); and (4) gender-marked occupations (e.g., ``waitress'') and gender-exclusive organizations (e.g., fraternities or sororities). Because applicants' home countries often feature prominently in their application materials, we replace their name with that of a randomly chosen historical applicant from the same country with the appropriate gender. Further details, including our validation process and prompt templates, appear in Appendix~\ref{app:audit}.

By Eq.~\eqref{eq:bias_synth_decomp}, \(\D {bias}_T(h)\) equals the baseline bias in \(h\) plus twice the product of \(b\) and \(\B E[\pi(T(X, m)) - \pi(T(X, f))]\), where the latter term quantifies the extent to which our representations encode gender. In our data, we estimate this difference is \(0.91 \pm 0.01\), meaning that altering applications to indicate gender strongly impacts the embeddings, which in turn impacts the model's predictions of gender. Relatedly, ridge regression models predicting self-reported gender from the text embeddings achieve an out-of-sample AUC of \(99.6\% \pm 0.3\%\), again demonstrating that the embeddings encode gender with near-perfect accuracy. While this empirical result is not surprising given how well standard black-box text embeddings encode information, it does suggest that the models will strongly inherit the bias injected into the labels.

Further, the baseline \(\D {bias}_T(\mu)\) equals \(-0.64 \pm 0.56\), meaning that the algorithm is biased \emph{against} men when predicting the ground-truth label. Proposition~\ref{prop:baseline_bias} provides one potential explanation: black-box embeddings fail to capture certain features relevant to admissions scores that are also correlated with gender. Gender subsequently serves as a proxy for that missing information, inducing the bias we observe with unbiased labels. We note that this quantity represents the bias of the \emph{algorithm} and not the bias of the admissions \emph{decisions}. Consistent with this interpretation, predictive models trained using a broader collection of features, including information about past student performance not well-captured by black-box embeddings, exhibit virtually no bias, suggesting that the actual admissions decisions are not in fact biased; see Fig.~\ref{fig:audit_all}.

The left panel of Fig.~\ref{fig:audit_study} shows that our theoretical estimates of \(\D {bias}_T(h)\) (dashed line)  broadly track the increase in empirical audit bias as the male advantage grows.
\footnote{%
  The agreement is especially close because ridge regression models vary linearly in their training labels, so an analogue of Prop.~\ref{prop:bias} holds almost exactly even in finite samples; see Prop.~\ref{prop:bias_ridge} in Appendix~\ref{app:math}. For regression models that are not exactly linear in their labels, such as random forests, the agreement can be looser; see Fig.~\ref{fig:app_robustness_models}, middle row.
}
The middle and right panels illustrate the consequences for admissions. We consider top-\(k\) admissions policies that rank applicants by predicted score and admit the top 20\%, a hypothetical admissions rate in line with the admissions policies of competitive graduate programs. As the advantage given to male applicants grows, the proportion of women admitted falls precipitously (middle panel), dropping to effectively zero when \(b = 2.5\). The quality of the admitted class, measured by applicants' ``ground-truth'' scores \(Y\), declines in parallel (right panel), as biased predictions substitute less qualified men for more qualified women.

\begin{figure}[t]
  \begin{center}
    \includegraphics{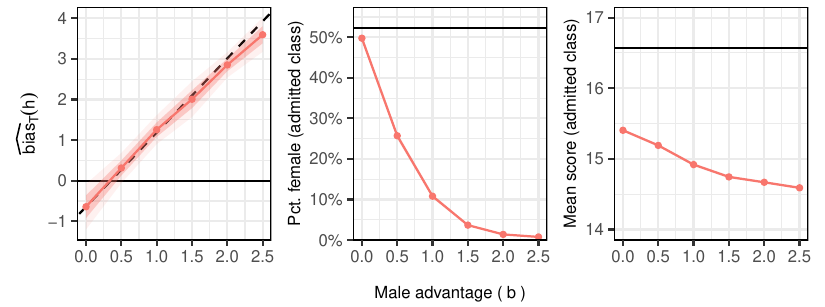}
  \end{center}
  \vspace{-0.25cm}
  \caption{\emph{%
    Bias of regression models trained on proxy labels \(Y'\) and consequences for admitted classes. The \(x\)-axis in all panels denotes male advantage \(b\) in Eq.~\eqref{eq:biased_outcome}. Solid black lines show zero bias (left) or corresponding values for the top 20\% of students ranked by actual score \(Y\) (center, right). \emph{Left:} Estimated bias \(\widehat{\D {bias}}_T(h)\) of fitted models \(h(x)\). The dashed line shows bias estimated according to Eq.~\eqref{eq:bias_decomp}. \emph{Middle:} Percentage of women in the admitted class under a top-20\% policy. \emph{Right:} Mean true score \(Y\) of the admitted class. Dark and light shading show $\pm1$ and $\pm2$ standard errors, respectively.% 
  }}
  \label{fig:audit_study}
\end{figure}

\section{Mitigating Label Bias}

Our extended example above illustrates the propagation of label bias: black-box embeddings encode protected attributes such as gender, which in turn allows models trained on biased labels to reproduce those biases. Our theoretical analysis suggests this phenomenon stems in part from the ability of black-box embeddings to capture a rich set of signals in the data, including those related to gender. This suggests choosing representations that retain decision-relevant information while limiting gender signals.  We first analyze three common approaches in the algorithmic fairness literature---orthogonalization, redaction, and marginalization---and discuss their limitations. We then introduce rubric embeddings, showing that they effectively eliminate label bias in our setting while avoiding the shortcomings of other methods.

\subsection{Orthogonalization, redaction, and marginalization}
\label{sec:ortho}

Orthogonalization~\citep{bolukbasi2016man, ravfogel2020null} aims to
remove group information from a representation \(R=r(X)\), ideally
achieving \(R \indep G\).
Orthogonalization and the closely related technique of adversarial debiasing~\citep{edwards2015censoring, zhang2018mitigating} are common recommendations in the literature for achieving fair downstream predictions when training on rich features that can closely approximate group membership. Both techniques proceed by removing information about a protected characteristic like gender from individuals' feature representations while otherwise leaving the remaining information intact, for instance by projecting latent representations of applicants onto subspaces orthogonal to their gender. The end goal of orthogonalization is to eliminate gender from an individual's features in an information-theoretic sense.

This approach, however, comes with substantial drawbacks. In particular, independence from the protected characteristic \(G\) carries through to downstream decisions based on the latent representation \(R\). For instance, if \(d(r) : \C R \to \{0, 1\}\) represents an admissions policy, the proportion of admitted students from different groups will necessarily equal the population proportion, i.e., 
\[
\Pr(G = g \mid d(R) = 1) = \Pr(G = g).
\]
(See Appendix~\ref{app:math} for proof.)
But if one group of applicants is more qualified than another---for instance, if female applicants are, on average, more qualified than male applicants---enforcing such demographic parity can inadvertently penalize that more qualified group. 
While in practice orthogonalization may only partially remove group information, even approximate independence can induce substantial shifts toward demographic parity.
Indeed, while the aim of orthogonalization is to improve equity, it might itself constitute illegal gender (or racial) ``balancing''~\citep{2023students}; though cf.~\citep{2023coalitionTJ, 2023bostonParent}.

Another natural approach to reducing bias is to redact signals of gender from an applicant's materials \emph{before} embedding them. Formally, redaction represents a map \(\rho : \C X \to \C X\) with the property that it masks the transformation \(T\), i.e.,
\begin{equation}
\label{eq:redaction}
    \rho(T(X, g)) = \rho(X), \qquad g \in \{m, f\}.
\end{equation}
In our applied admissions example, for instance, we operationalize \(\rho(x)\) by creating alternate gender-neutral representations of applicants' materials in which we obscure signals of gender instead of swapping their gender presentation (e.g., replacing ``he'' with ``they'' instead of ``her''); see Appendix~\ref{app:audit} for full details. 

Redaction is guaranteed to be unbiased with respect to \(T\) (see Prop.~\ref{prop:no_bias}). But unlike orthogonalization, the elimination of gender information through redaction is narrowly tied to a particular manipulation, potentially limiting its effectiveness. To evaluate this possibility, we train models predicting applicants' genders using black-box embeddings of the redacted materials. Despite our redaction efforts, we find that these models still predict applicants' genders with high accuracy, achieving an AUC of \(81.3\% \pm 2.6\%\). Even in the absence of strong gender signals---like names and pronouns---the predictive models we train can nevertheless ascertain individuals' gender based on subtle statistical regularities encoded by the black-box embeddings, a finding consistent with past work~\citep[e.g.,][]{wen2025unsupervised, gaebler2025auditing, parasurama2022gendered}. As a result, predictive models trained on black-box embeddings of facially gender-neutral materials can nevertheless reproduce label bias in their predictions, a point we return to below.

Finally, marginalizing out gender or other protected characteristics is a well-explored technique in economics~\citep[e.g.,][]{yang_equal_2020, pope_implementing_2011}, as well as the causal fairness literature~\citep[e.g.,][]{wang2019equal}. Marginalized predictions replace an individual's actual prediction with a weighted sum of the predictions made when their actual features are replaced with features belonging to different groups. Formally, marginalization replaces an algorithm \(h(x) : \C X \to \B R\) with the weighted sum
\begin{equation}
\label{eq:marginal}
    h_T(x) \defeq w \cdot h(T(x, m)) + [1 - w] \cdot h(T(x, f)),
\end{equation}
for some weight \(w \in [0, 1]\). 

Like redaction, this approach results in predictive algorithms with no bias with respect to \(T\) (see Prop.~\ref{prop:no_bias}).
However, despite its intuitive appeal, marginalization can lead to poor outcomes in practice when implemented using black-box embeddings, which we empirically illustrate below.

\subsection{Rubric Embeddings}
\label{sec:rubric_embeds}
Orthogonalization, redaction, and marginalization share a common methodological goal: surgically eliminating information tainted by gender or other protected characteristics while preserving as much as possible of the rich information present in black-box embeddings. \emph{Rubric embeddings} represent the inverse approach. Rather than removing potentially problematic information, rubric embeddings consist of only information deemed substantively important for the prediction task. Conceptually, rubric embeddings restrict the hypothesis class to functions of semantically meaningful features, limiting the ability of models to exploit spurious correlations. In domains like admissions, hiring, and medicine, expert decision makers can often articulate what they are looking for---even if they cannot specify a precise decision rule based on those dimensions. Enumerating these dimensions using a rubric against which unstructured materials can be scored (e.g., by grading them with a language model) allows the construction of a feature representation anchored to the aspects of the materials of greatest substantive relevance.

To demonstrate this approach, in collaboration with the admissions staff at our partner master's program, we develop a collection of 396 interpretable measures admissions officers deem substantively important in application review. See Appendix~\ref{app:rubrics} for detailed information regarding the rubric construction process and final feature set.  

We omit applicant demographics from our rubric embedding. Building on the document extraction pipeline used to extract black-box embeddings, we use OpenAI's GPT-5 series of models~\citep{openai2025gpt5} to score the full applicant pool along each of the 396 dimensions.\footnote{%
 Fig.~\ref{fig:feature_shift} reports average paired differences in re-extracted essay, resume, and recommendation-letter scores under the gender manipulation. Transcript-derived features and structured covariates are held fixed. 
} 

The rubric embeddings are based largely on objectively measurable factors and do not include explicit indicators of gender. As such, for our usual transformation \(T\), we expect that \(r(T(X, g)) = r(X)\), implying that models based on rubric embeddings should be unbiased with respect to \(T\).  Below we empirically check this and examine the demographic composition and quality of the applicants models using rubric embeddings rank highest. We further compare this rubric embedding approach to all of the approaches discussed above: (1) predictive models trained on black-box embeddings in Section~\ref{sec:problem} (``black-box embeddings''), (2) predictive models trained using both black-box and rubric embeddings (``kitchen sink''), (3) predictive models trained on black-box embeddings of redacted application materials (``redacted embeddings''), (4) marginalized black-box embedding models (``marginalized''), (5) black-box embedding models in which each feature is uncorrelated with gender (``orthogonalized''). Although orthogonalization may itself raise legal concerns in jurisdictions that prohibit disparate treatment, we include it for completeness. 

\begin{figure}[t]
  \begin{center}
    \includegraphics{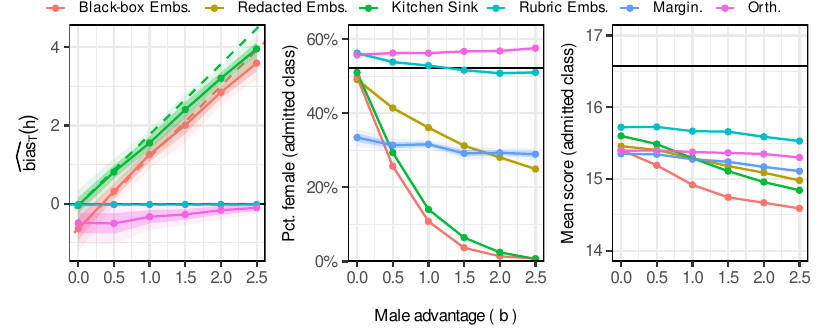}
  \end{center}
  \caption{\emph{%
    Bias of models trained on proxy labels \(Y'\) under different bias mitigation techniques and consequences for admitted classes. The \(x\)-axis in all panels denotes male advantage \(b\) in Eq.~\eqref{eq:biased_outcome}. Colors indicate the bias mitigation technique. Dark and light shading show $\pm1$ and $\pm2$ standard errors, respectively. Solid black lines show zero bias (left) or corresponding values for the top 20\% of students ranked by actual score \(Y\) (center, right). \emph{Left:} Estimated bias \(\widehat{\D{bias}}_T(h)\) of fitted models \(h(x)\). Dashed line shows bias estimated according to Eq.~\eqref{eq:bias_decomp}. \emph{Middle:} Percentage of women in the admitted class under a top-20\% policy. \emph{Right:} Mean true score \(Y\) of the admitted class. %
  }}
\label{fig:audit_all}
\end{figure}

The audit bias and selected class composition results are shown in Fig.~\ref{fig:audit_all}. As the left panel demonstrates, across levels of male advantage, we find virtually zero bias in rubric embedding models, consistent with their weak representation of gender (rubric-embedding models poorly predict gender, with an AUC of \(61.7\% \pm 3.6\%\)). We omit marginalized and redacted embedding models, which are guaranteed to have negligible bias. Consistent with our discussion in Section~\ref{sec:ortho}, we find orthogonalized models admit a lower-scoring cohort than rubric embedding models and exhibit moderate bias. This is likely due to the fact that orthogonalization removes gender information from the general-purpose embeddings and pushes decisions toward demographic parity, disregarding decision-relevant compositional differences between male and female applicants in our pool. The center and right-hand panels of Fig.~\ref{fig:audit_all} illustrate that rubric embeddings likewise admit the most gender-balanced and highest quality classes, across the full range of male advantage. We also evaluate selected cohorts using prior GPA and GRE scores (Fig.~\ref{fig:audit_covariates}). Rubric embeddings perform competitively across each of these measures, though they do not uniformly outperform the alternatives. 
Our results are robust to a variety of alternative specifications, including alternative model families, embedding models, different ground-truth labels and dimensions of advantage, and subsets of the rubric features. See Appendix~\ref{app:robustness} for further details.

While generally more performant than the other bias mitigation techniques, marginalized models produce admissions policies that admit substantially fewer women and a much less academically prepared class than rubric embedding models.
This pattern is a consequence of the technique's forced bias correction. As shown in Fig.~\ref{fig:audit_study}, the black-box embedding models exhibit a slight bias against male candidates when \(b = 0\). Marginalizing out gender penalizes the group that would otherwise have an advantage, and \emph{vice versa}. In this case, that means the model penalizes women, leading to fewer women among the top-ranked applicants. Importantly, though, the black-box model's bias against men does not appear to reflect actual bias in the ground-truth labels \(Y\). Indeed, after adjusting for our rubric items, we find no significant gender disparity in admissions scores, with a male-female score difference of 0.17 (95\% CI [-0.07, 0.41]). This suggests that the bias in the black-box models results from their particular representation. Consequently, while the marginalized models attempt to ``correct'' apparent bias in the black-box models, they ultimately unjustly penalize female applicants.

Why do rubric embeddings perform so well relative to other techniques? \citet{label-bias} show that unlike when predicting ground-truth labels, including additional covariates can reduce prediction accuracy in the presence of label bias. In particular, they show that the relative predictive accuracy of models using a reduced and full set of features is controlled by three terms: (1) how accurately models trained on the reduced feature set can predict the proxy label, (2) how accurately models trained on the full feature set can predict the proxy label, and (3) how correlated the ground truth label is with the predictions using the full feature set, conditional on the reduced features. Since the models trained on the full set of covariates in expectation will predict the proxy label weakly better than models trained on the reduced set, the operative term is the correlation. The DAG shown in the left-hand panel of Fig.~\ref{fig:rmse} approximately describes our admissions setting, where admissions officers attempt to score applicants based essentially entirely on the information captured in our rubric embeddings. In this stylized model, because rubric embeddings \(d\)-separate the black-box embeddings and the ground truth, conditional on the rubric embeddings, any function of the black-box embeddings---\emph{a fortiori}, any predictions based on them---should be uncorrelated with the true labels. This is precisely what we observe: the center and right-hand panels of Fig.~\ref{fig:rmse} illustrate that while kitchen sink and black-box models achieve lower RMSE than the rubric embedding models relative to the \emph{proxy} label, relative to the ground truth, rubric embedding models achieve the highest accuracy. (Fig.~\ref{fig:diff_in_mse} shows 
this agreement between theoretical and empirically predicted values of the MSE gap.)

\begin{figure}[t]
    \centering
    \begin{subfigure}[c]{0.32\textwidth}
        \centering
        \fbox{%
        \begin{tikzpicture}[
            node distance=0.66cm and 0.42cm,
            every node/.style={
                draw, rounded corners,
                align=center, inner sep=3pt,
                font=\footnotesize,
                minimum height=0.45cm,
                minimum width=1cm
            },
            arr/.style={-Latex, semithick},
            dasharr/.style={dotted, semithick}
        ]
            \node (gender)      {Gender};
            \node (application) [right=of gender]      {Application};
            \node (latent)      [above=of application]  {Black-box};
        
            \node (biased) [below=of gender]      {Proxy label};
            \node (rubric) [below=of application] {Rubric};
            \node (true)   [below=of biased]      {True label};
        
            \draw[dasharr] (gender)      -- (application);
            \draw[arr]     (application) -- (latent);
            \draw[arr]     (application) -- (rubric);
            \draw[arr]     (gender)      -- (biased);
            \draw[arr]     (rubric)      -- (true);
            \draw[arr]     (true)        -- (biased);
        \end{tikzpicture}
        }
        \vspace{0.7cm}
    \end{subfigure}
    \begin{subfigure}[c]{0.67\textwidth}
        \includegraphics{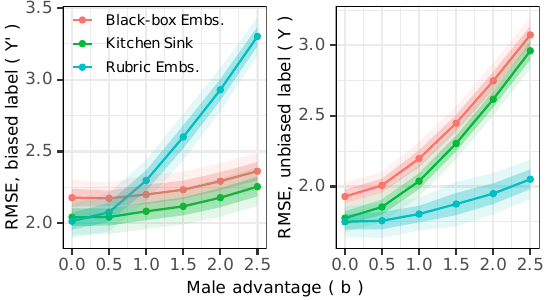}
    \end{subfigure}
  \caption{\emph{%
    Theoretical explanation and empirical verification of rubric embedding models' superior predictive performance. \emph{Left:} A causal DAG representing our admissions setting. The dotted line indicates correlation. \emph{Center:} RMSEs of rubric embedding, black-box embedding, and kitchen sink models, evaluated relative to proxy labels \(Y'\). \emph{Right:} RMSEs of rubric embedding, black-box embedding, and kitchen sink models, evaluated relative to ground truth \(Y\). At \(b=0\), proxy labels retain zero-mean, unit-variance noise, so RMSEs against \(Y'\) and \(Y\) need not coincide. Dark and light shading show $\pm1$ and $\pm2$ standard errors, respectively.
  }}
\label{fig:rmse}
\end{figure}

\section{Discussion}
\label{sec:discussion}
Label bias is a long-standing and pernicious problem in designing equitable algorithms. But despite its importance, there are few general techniques available to address it---short of collecting less biased labels~\cite{obermeyer2019dissecting}. 
Rubric embeddings provide a practical approach to addressing this challenge.
Our analytical and empirical results demonstrate that rubric embeddings can yield more accurate and equitable decisions.
Our results further suggest that bias arising from proxy labels is not only a property of the labels themselves, but also of the representations used to model them. There are, however, important limitations of our approach, which we discuss below.

First, constructing a rubric requires substantial time and domain expertise. The ideal rubric includes all legitimate decision-relevant factors and excludes irrelevant ones. In our admissions setting, we developed the rubric through iterative engagement with admissions officers, including reviewing evaluation guidelines, observing evaluation processes, and participating in admissions committee discussions. We then used LLMs to operationalize these criteria into a structured representation that could be consistently applied to applications, followed by manual review to resolve discrepancies between human and algorithmic assessments. While this process is difficult---and can only partially be automated---we believe it is critical for developing a high-quality rubric. In the future, more capable models may be able to increasingly facilitate or automate certain aspects of rubric design, though this may come at the expense of substantive discussion among decision-makers about which factors should be considered in the first place.

Second, basing predictions on rubric embeddings mitigates some forms of bias but not all. By construction, rubric embeddings enumerate factors that are considered legitimate for decision-making, and exclude group membership. As a result, they offer strong protection against \emph{disparate treatment}, the form of discrimination typically studied in correspondence experiments. However, rubric embeddings cannot erase genuine group differences in decision-relevant qualifications, and thus  even when the rubric only contains decision-relevant information, label bias can still distort the weights assigned to those factors. For example, if female applicants are more likely to engage in public service, then labels biased against women may lead to an inappropriately low weight on public service, thereby penalizing qualified women. This form of \emph{disparate impact} is more difficult to eliminate. In practice, rubric embeddings mitigate this risk by constraining the space of admissible models to those based on semantically meaningful factors. In Appendix~\ref{app:robustness}, we explore this possibility by simulating label bias rewarding undergraduate economics, business, and engineering majors---a majority of whom are men---and find rubric embedding-based models show no audit-study bias and admit more gender-balanced classes than black-box embedding-based models. However, the quality of the admitted class declines as the advantage grows, since the rewarded features are directly represented in the rubric. More fundamentally, disparate impact is difficult to assess in this context because there is no consensus on which factors admissions should reward. For example, whether to weight leadership experience heavily is a context-dependent judgment.

Third, by design, rubric embeddings exclude group membership and related features from the representation. While this provides protection against disparate treatment, it also limits the ability of models to incorporate group-specific information, even in settings where such information may be relevant or normatively justified. For example, in some medical contexts, group membership can be predictive of risk due to underlying structural or biological differences, and incorporating such information may improve outcomes~\citep{coots2023reevaluating}. Similarly, in policy settings such as affirmative action, decision-makers may wish to explicitly account for group membership to address historical inequities. By restricting attention to group-neutral representations, rubric embeddings may preclude these types of interventions. This limitation reflects a fundamental tension between enforcing group-blind decision rules and allowing for context-dependent uses of group information.

Finally, our empirical analysis is limited to the singular empirical setting of admissions to a master's program. However, our population decomposition in Prop.~\ref{prop:bias} is domain agnostic and does not impose a parametric form on the conditional expectation. We therefore expect similar patterns to emerge wherever experts can clearly articulate evaluation criteria (e.g., hiring and language model alignment~\citep{gaebler2026constitutions}) though confirming this empirically is important future work.

Despite these limitations, rubric embeddings offer a practical approach for mitigating label bias in a variety of high-stakes settings, from university admissions to hiring. More broadly, our results suggest that addressing bias requires not only improving data and learning algorithms, but also reconsidering how inputs are represented. By grounding predictions in semantically meaningful, domain-specific criteria, rubric embeddings provide a principled way to align learned models with the underlying construct of interest, offering a promising direction for designing more reliable and
equitable decision-making systems.

\begin{ack}
    This work received no external funding. The authors declare no competing interests.
\end{ack}

\bibliography{ref.bib}

%%%%%%%%%%%%%%%%%%%%%%%%%%%%%%%%%%%%%%%%%%%%%%%%%%%%%%%%%%%%

\newpage
\appendix

\setcounter{figure}{0}
\renewcommand{\thefigure}{A\arabic{figure}}
\setcounter{thm}{0}
\renewcommand{\thethm}{A\arabic{thm}}
\setcounter{prompt}{0}
\renewcommand{\theprompt}{A\arabic{prompt}}

\section{Mathematical Appendix}
\label{app:math}

\paragraph{Conventions and non-degeneracy assumptions.}

For an integrable random variable \(A\) and random object \(W\), we write \(\B E[A \mid W = w]\) for a chosen measurable version of the conditional expectation as a function of \(w\); by \(\B E[A \mid Z = \zeta(w)]\) we mean the composition of the maps \(\B E[A \mid Z = z]\) and \(\zeta(w) : \C W \to \C Z\).

To avoid trivial degeneracies arising from the choice of version, we make the following simple regularity assumptions. First, the probability of belonging to either gender is positive, i.e.,
\begin{equation}
\label{eq:assump_G}
    0 < \Pr(G = m) < 1.
\end{equation}
Second, we assume that the distributions of \(T(X, m)\) and \(T(X, f)\) are absolutely continuous with respect to the distribution of \(X\), i.e.,
\begin{equation}
    \text{if} \quad \Pr(X \in E) = 0 \qquad \text{then} \quad \Pr(T(X, g) \in E) = 0, \quad g \in \{m, f\}.
\end{equation}
All functions are assumed to be measurable and all random variables integrable. All equalities between random variables should be understood as almost sure equality.

\paragraph{Proofs.}

We begin with the proof of Prop.~\ref{prop:bias}.

\begin{proof}[Proof of Prop.~\ref{prop:bias}]
    By linearity and the tower property
    \begin{align*}
        \nu(X)
            &= \B E[Y \mid r(X)] + \B E[B \mid r(X)] \\
            &= \B E[Y \mid r(X)] + \sum_{g \in \{m, f\}} \B E[B \mid r(X), G = g] \cdot \Pr(G = g \mid r(X)).
    \end{align*}
    Since \(B \indep X \mid G\), \(B \indep r(X) \mid G\), and so \(\B E[B \mid r(X), G = g] = \B E[B \mid G = g]\). Define \(\beta(g) \defeq \B E[B \mid G = g]\). Then
    \[
        \nu(X) = \mu(X) + \beta(f) + [\beta(m) - \beta(f)] \cdot \pi(X),
    \]
    where we have used the fact that \(\Pr(G = f \mid r(X)) = 1 - \pi(X)\).

    Substituting this expression into \(\B E[\nu(T(X, m)) - \nu(T(X, f))]\) yields
    \[
        \B E[\mu(T(X, m)) - \mu(T(X, f))] + [\beta(m) - \beta(f)] \cdot \B E[\pi(T(X, m)) - \pi(T(X, f))],
    \]
    which is Eq.~\eqref{eq:bias_decomp}.
\end{proof}

Proposition~\ref{prop:bias} characterizes the behavior of a predictive algorithm in the limit. In principle, in finite samples, predictive algorithms can deviate from Eq.~\eqref{eq:bias_decomp}. However, for ridge regression, the empirical analogue of Proposition~\ref{prop:bias} holds almost  exactly even in finite samples, as we observe in Figures~\ref{fig:audit_study} and~\ref{fig:audit_all}.

\begin{prop}
\label{prop:bias_ridge}
    Let \(\bb R \in \B R^{n \times p}\) be the matrix with rows \(r(X_1), \ldots, r(X_n)\), centered so that \(\bb R^\top \bb 1 = 0\), and let \(\bb Y, \bb B \in \B R^n\) and \(\bb G \in \{m, f\}^n\). Let \(\bb R_m, \bb R_f \in \B R^{N \times p}\) be the matrices with rows \(r(T(\tilde x_j, m))\) and \(r(T(\tilde x_j, f))\) for \(\tilde x_1, \ldots, \tilde x_N \in \C X\). For \(\lambda > 0\), define for \(\mathbf h \in \B R^p\),
    \[
        \hat h(\bb y) \defeq (\bb R^\top \bb R + \lambda I)^{-1} \bb R^\top \bb y,
        \qquad
        \widehat{\D {bias}}_T(\mathbf h) \defeq \frac{1}{N} \left[ \bb 1^\top (\bb R_m - \bb R_f) \mathbf h \right].
    \]
    Let \(\hat \beta(g)\) be the average of \(B_i\) over \(i\) such that \(G_i = g\), and suppose \(\bb R^\top \boldsymbol \varepsilon = 0\), where \(\varepsilon_i \defeq B_i - \hat \beta(G_i)\). Then
    \[
        \widehat{\D {bias}}_T(\hat h(\bb Y + \bb B)) = \widehat{\D {bias}}_T(\hat h(\bb Y)) + [\hat \beta(m) - \hat \beta(f)] \cdot \frac{1}{N} \sum_{j=1}^N [\hat \pi(T(\tilde x_j, m)) - \hat \pi(T(\tilde x_j, f))],
    \]
    where \(\hat \pi(x) \defeq r(x)^\top \hat h(\bb 1_m)\) and \(\bb 1_m\) is the indicator vector of \(G_i = m\).
\end{prop}

Proposition~\ref{prop:bias_ridge} states that for the ridge regression \(\hat h(\bb y)\) fit with design matrix \(\bb R\) to labels \(\bb y\), the empirical audit bias on a (possibly distinct) collection of applications \(\tilde x_1, \ldots, \tilde x_N\) decomposes exactly as in Eq.~\eqref{eq:bias_decomp}, with \(\pi\) replaced by a linear ridge probability model predicting gender and fit to the same features and the independence condition replaced by orthogonality. The proof is almost identical to Proposition~\ref{prop:bias}.

\begin{proof}[Proof of Prop.~\ref{prop:bias_ridge}]
    By definition, \(\bb B = \hat \beta(f) \bb 1 + [\hat \beta(m) - \hat \beta(f)] \cdot \bb 1_m + \boldsymbol \varepsilon\). Since \(\hat h(\bb y)\) is linear in \(\bb y\), \(\hat h(\bb 1) = 0\) because \(\bb R\) is centered, and \(\hat h(\boldsymbol \varepsilon) = 0\) because \(\bb R^\top \boldsymbol \varepsilon = 0\). Therefore
    \[
        \hat h(\bb Y + \bb B) = \hat h(\bb Y) + [\hat \beta(m) - \hat \beta(f)] \cdot \hat h(\bb 1_m).
    \]
    Since \(\widehat{\D {bias}}_T(\mathbf h)\) is linear in \(\mathbf h\), and \(\widehat{\D {bias}}_T(\hat h(\bb 1_m)) = \frac{1}{N} \sum_j [\hat \pi(T(\tilde x_j, m)) - \hat \pi(T(\tilde x_j, f))]\) by the definition of \(\hat \pi\), this gives the result.
\end{proof}

Proposition~\ref{prop:baseline_bias} follows as a straightforward consequence of the linearity of expectation.

\begin{proof}
    By the definition of \(h_R\), we have that \(\D {bias}_T(h_R)\) is
    \[
         \B E[\alpha_0 + r(T(X,m))^\top\alpha - \alpha_0 - r(T(X,f))^\top\alpha] = \B E[r(T(X, m)) - r(T(X, f))]^\top \alpha,
    \]
    where the equality follows from the linearity of expectation.
    The latter expression is \(s^\top \alpha\). Since \(\alpha = \D{Var}(R)^{-1}\D{Cov}(R, Y)\), it follows that
    \[
        \D{bias}_T(h_R) = s^\top \D{Var}(R)^{-1}\D{Cov}(R, Y),
    \]
    which is Eq.~\eqref{eq:baseline_bias}.
\end{proof}

We next state and prove a straightforward but important consequence of orthogonalization.

\begin{lem}
\label{lem:ortho_dem_par}
    If \(R \indep G\) and \(\Pr(d(R) = 1) > 0\), then \(\Pr(G = g \mid d(R) = 1) = \Pr(G = g)\).
\end{lem}

\begin{proof}
    Since \(R \indep G\), it follows that any \(\sigma(R)\)-measurable function is independent of \(G\), whence \(d(R) \indep G\), i.e., \(\Pr(G = g \mid d(R) = 1) = \Pr(G = g)\) for \(g \in \{m, f\}\).
\end{proof}

Both redaction and marginalization result in decision algorithms that are unbiased. In the case of marginalization, this depends on an intuitive consistency condition for the manipulation \(T\).
Note that manipulations \(T(x, g)\) for different \(g\) typically mirror one another, altering names, pronouns, or other elements of \(x\) in parallel.
Consequently, the result of applying a manipulation \(T\) multiple times is typically the same as the result of the final application:
\begin{equation}
\label{eq:consistency}
    T(T(X, g_0), g_1) \overset {d} {=} T(X, g_1), \qquad g_0, g_1 \in \{m, f\}.
\end{equation}
Distributional equality accounts for stochasticity in the manipulation.

\begin{prop}
\label{prop:no_bias}
    With the same notation as in Prop.~\ref{prop:bias}, let \(h(x) : \tilde {\C X} \to \B R\) be arbitrary.
    \begin{itemize}
        \item \textbf{Redaction}: If \(\rho\) and \(T\) satisfy Eq.~\eqref{eq:redaction}, then \(\D {bias}_T(h \circ \rho) = 0\).
        \item \textbf{Marginalization}: If \(T\) satisfies Eq.~\eqref{eq:consistency}, then \(\D {bias}_T(h_T) = 0\).
    \end{itemize}
\end{prop}

\begin{proof}
    First, we consider \emph{redaction}. Observe that
    \[
        \D {bias}_T(h \circ \rho) = \B E[h(\rho(T(X, m))) - h(\rho(T(X, f)))] = \B E[h(\rho(X)) - h(\rho(X))] = 0
    \]
    where the final equality follows from Eq.~\eqref{eq:redaction}.

    Next, we consider \emph{marginalization}. Substituting Eq.~\eqref{eq:marginal} into Def.~\ref{defn:bias} gives that
    \[
        \D {bias}_T(h_T) = \B E[h_T(T(X, m)) - h_T(T(X, f))].
    \]
    Expanding using Eq.~\eqref{eq:marginal}, this becomes
    \begin{multline*}
        \B E \big[ w \cdot h(T(T(X, m), m)) + (1 - w) \cdot h(T(T(X, m), f)) \\
            - w \cdot h(T(T(X, f), m)) - (1 - w) \cdot h(T(T(X, f), f)) \big],
    \end{multline*}
    which, by the linearity of expectation, equals
    \begin{multline*}
        w \cdot \B E[h(T(T(X, m), m))] + (1 - w) \cdot \B E[h(T(T(X, m), f))] \\
            - w \cdot \B E[h(T(T(X, f), m))] - (1 - w) \cdot \B E[h(T(T(X, f), f))].
    \end{multline*}
    Recalling Eq.~\eqref{eq:consistency}, this in turn equals
    \[
        w \cdot \B E[h(T(X, m))] + (1 - w) \cdot \B E[h(T(X, f))] - w \cdot \B E[h(T(X, m))] - (1 - w) \cdot \B E[h(T(X, f))],
    \]
    which is zero.
\end{proof}

\clearpage

\section{Creating the rubric embeddings}
\label{app:rubrics}

While making admissions decisions, the officers at the school we study use a scoring guide to calibrate where they place applicants on the 4--20 scale. Fig.~\ref{fig:score_dist} shows the distribution of these scores. The scoring guide decomposes the final score into several subcategories, and describes in detail what separates applicants receiving each point value within each category, with reference to specific application features. Our original rubrics were based off of this scoring document: for each type of unstructured data applicants provide (e.g., transcripts) we identify where in the scoring guide that document is mentioned, and what information from that document is considered (e.g., number of prior quantitative courses). 

\begin{figure}[]
    \begin{center}
        \includegraphics{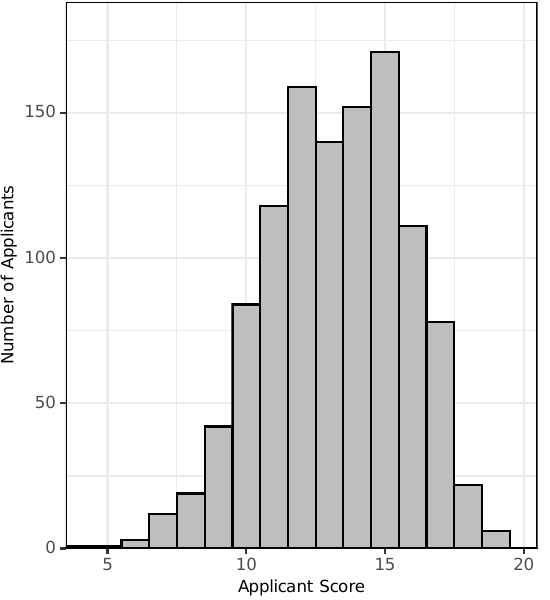}
    \end{center}
    \caption{\emph{%
        The distribution of actual applicant scores in the 2025--2026 round of admissions.
    }}
\label{fig:score_dist}
\end{figure}

Having enumerated each of these decision-relevant variables (or, the information necessary to construct them post-hoc) we then, for each type of material (resume, transcript, letter of recommendation, essay) prompt \texttt{gpt-5-mini}~\citep{openai2025gpt5} to extract this information from the applicant's provided material(s), and return the result in a structured format with \texttt{gpt-5-mini}'s structured output functionality. We provide \texttt{gpt-5-mini} with a markdown version of the applicant's material, submitted as a PDF, converted using Microsoft Document Intelligence~\citep{microsoft2024documentintelligence}, along with the rubric and the expected structured output format. The resulting output is then converted to a structured format, aggregated to the applicant level, and included in our rubric embedding models. For further information regarding our document processing pipeline, see the appendices of concurrent work~\citep{isley2026aiwrittenadmissionsessayswidespread, gaebler_test_scores}.

Crucially, this process involved multiple rounds of iteration with admissions officers, observing their decision making process, and exploring what decision-relevant information we had missed when our preliminary model scores dramatically differed from their judgment. This iterative process grants us the 396 features we mention above. Of these 396 features, 14 are extracted from the applicant's essays, covering things such as grammatical correctness, clarity of writing, prompt adherence and substantive appropriateness. Eight measures come from resumes, summarizing applicants' work histories, including tenure in different sectors along with measures of seniority and career trajectory. We extract 39 measures from an applicant's letters of recommendation, including the strength of the recommender's endorsement (both overall and along a variety of personal dimensions), references to specific program-relevant technical skills, and substantive appropriateness of the choice of recommender. Finally, the remaining 335 distinct measures are extracted from transcripts, covering previous academic performance, including majors and minors, graduation honors, rank and other background information about previous institutions, and GPA both overall and in specific courses (e.g., microeconomics, calculus II), and areas (e.g., advanced undergraduate mathematics and statistics). Applicants were required to submit five essays and three letters of recommendation, and many submitted both undergraduate and graduate transcripts. For essays and letters of recommendation, we take averages across each document-level extraction to aggregate to the applicant level. For transcripts, we encode each transcript dimension for undergraduate and graduate school independently (but average within each level of school.) 

\section{Audit study}
\label{app:audit}

To create the materials necessary to conduct our audit study, we iteratively process each type of application material (essays, letters of recommendation, and resumes) using an LLM-based pipeline. We use OpenAI's \texttt{gpt-5-mini}~\citep{openai2025gpt5} for template construction and repair, and \texttt{gpt-5.4-mini}~\citep{openai2026gpt54thinking} for validation. (We do not create audit study materials for an applicant's transcripts, as the standardized text representations we create contain neither explicit references to gender nor applicants' names.)

\paragraph{Template construction.}

Our pipeline consists of three steps. First, we provide the language model with an application material's extracted markdown text, prompting it to return the otherwise unaltered full text with wrappers around all signals of gender and indications of an applicant's name in standardized templates formatted with ASCII control characters. For gender signals, these templates include placeholders with the original text, as well as the appropriate male, female, and gender-neutral alternatives. For example, the literal string \texttt{her} is transformed into \texttt{\textbackslash{}x02her\textbackslash{}x1Ftheir\textbackslash{}x1Fhis\textbackslash{}x1Fher\textbackslash{}x03}. For names, the model inserts a distinguished name placeholder, the field's original text, and the type of name indicator (e.g., email address, last name, social media handle). Under this name templating paradigm \texttt{@johndoe} becomes \texttt{@\textbackslash{}x02NAME\textbackslash{}x1Fhandle\textbackslash{}x1Fjohndoe\textbackslash{}x03}. See Prompt~\ref{prompt:template} for full details.

After a template has been generated, we use a different model to validate the correctness of the replacements. For the validation step, we prompt the model with the original prompt (Prompt~\ref{prompt:template}), the template to be validated, and an additional prompt instructing it to identify and similarly replace any uncaptured signals of gender and applicant name. If no omissions or errors are found, the validator returns an empty list. See Prompt~\ref{prompt:validation} for full details. In addition to LLM-driven validation checks, we confirm locally that the gender and name markers are properly formatted, and that the ``original'' variant is unchanged relative to the generating document using Python's \texttt{difflib}. Templates passing both LLM and local checks are considered complete. 

Finally, we attempt to repair any errors identified during the validation step. We prompt the model with the generated template, a list of all the issues raised during the validation step, and instruct it to output a template with the listed errors resolved. The full instructions are given in Prompt~\ref{prompt:repair}. The output is then re-validated. When creating the materials for our audit study, we repeatedly validate and repair documents until a small number (\(< .005\%\)) of uncompleted files remained for manual review. Our manual review consisted of either manual approval of model false-positives or human resolution of outstanding issues.

\paragraph{Generating transformed materials.}

Once a valid template has been constructed, we instantiate four variants---male, female, neutral, and ``original''---by filling in the appropriate gender and name placeholders. For gender, the original variant preserves the applicant's submitted language, and the male, female, and gender-neutral variants use the corresponding LLM-determined substitute. To appropriately swap in a new name for an applicant for their male and female variants, we compiled the names and home countries of historical applicants.
In the male and female variants of applicants' materials, they are assigned random names belonging to an applicant of the appropriate gender from the applicants' home countries. In the neutral material variants, we substitute ``FIRSTNAME LASTNAME,'' as appropriate, for name placeholders in the template and in the original, their true name. (Original variants are identical to the generating document and are created primarily to validate the template generation process.)

\paragraph{Generating transformed features.}

For each of the four document variants we generate for an applicant, we embed all the materials using OpenAI's \texttt{text-embedding-3-large}~\citep{openai2023textembedding}, yielding a collection of black-box embeddings for each variant. Similarly, we run our structured rubric extraction pipeline on each generated variant, producing rubric embedding features for all applicants across all conditions considered in our analysis. 

This process yields the transformations \(T(x, g)\) for \(g \in \{m, f\}\) required for our audit study, as well as the gender-neutral materials required to evaluate the redaction technique described in Section~\ref{sec:ortho}. 

Our approach assumes the use of an LLM extractor that scores the rubric criteria consistently and accurately. Fig.~\ref{fig:feature_shift} reports average paired differences in the re-extracted rubric features. Most estimated differences are close to zero, although some features exhibit detectable shifts. We do not replicate our rubric extraction with another LLM. 

\begin{figure}[]
  \centering
  \includegraphics{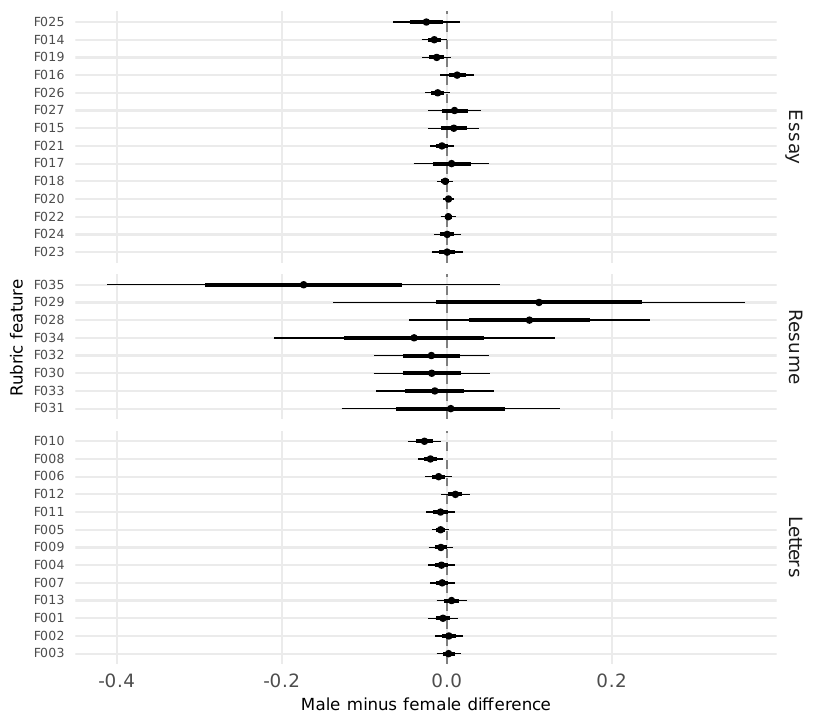}
  \caption{\emph{Mean paired differences between male- and female-presenting applications in re-extracted essay, resume, and recommendation-letter features. Points show mean differences; thick and thin bars show $\pm1$ and $\pm2$ standard errors, respectively. Feature identifiers are anonymized. Transcripts and structured covariates are held fixed and omitted.}}
  \label{fig:feature_shift}
\end{figure}

\begin{figure}[]
    \begin{center}
        \includegraphics{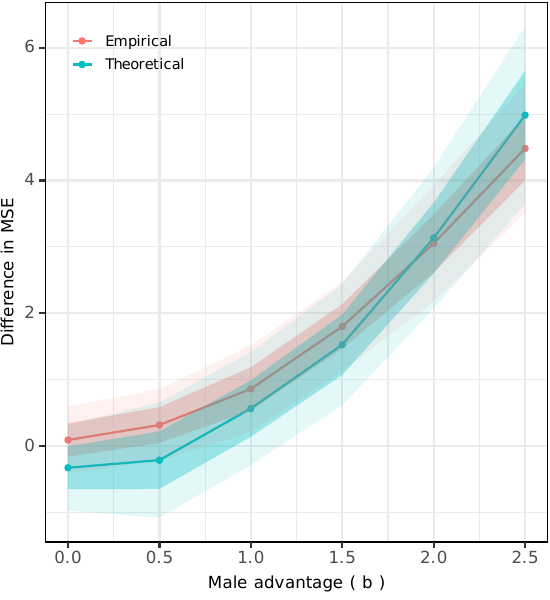}
    \end{center}
    \caption{\emph{%
        Theoretical predicted and empirical values of the gap in MSE between a kitchen sink and rubric embeddings model. 
    }}
\label{fig:diff_in_mse}
\end{figure}

\begin{figure}[]
  \begin{center}
    \includegraphics[]{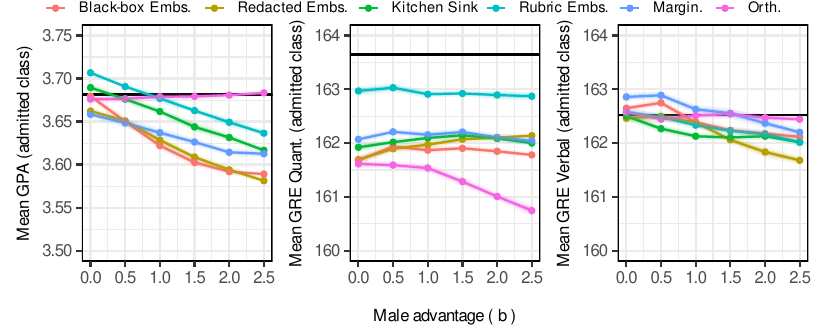}
  \end{center}
  \caption{\emph{%
    The impact on various measures of the strength of the admitted class using models of the proxy label with the corresponding level \(b\). Colors indicate the bias mitigation technique. Dark and light shading show $\pm1$ and $\pm2$ standard errors,
    respectively. Solid lines show corresponding values for the top 20\% of students admitted according to actual scores \(Y\). Left panel: The average GPA (graduate and undergraduate) of the admitted class.
    Center panel: The average GRE quantitative score of the admitted class. Right panel: The average GRE verbal score of the admitted class.
  }}
\label{fig:audit_covariates}
\end{figure}

\clearpage

\section{Robustness checks}
\label{app:robustness}
To confirm theory and empirical results hold outside of the gender bias case we study, here we conduct several robustness checks in which we vary several dimensions of our analysis. Overall, we find qualitatively similar results across each dimension of variation. Below, we provide further details regarding each of these robustness checks. Unless noted otherwise, each analysis follows the procedure of the main text: we inject advantage \(b \in [0, 2.5]\) into the labels, fit models using \(10 \times 10\)-fold nested cross-validation, and admit the top 20\% of applicants by predicted score.

\subsection{Modeling choices}

To assess whether our results are sensitive to different modeling choices, including statistical model family, choice of embedding model, choice of outcome or subset of rubric features, we reproduce our main text results with a variety of alternative specifications, detailed below in order. 

\paragraph{Model families}
To assess whether our results generalize beyond ridge regression, we additionally fit LASSO models and random forests. The top two rows of Fig.~\ref{fig:app_robustness_models} demonstrate our results are not specific to ridge regression. Both LASSO and random forest models follow similar patterns to our main ridge results: rubric-embedding models do not track the injected audit bias, while models including general purpose embeddings do. Correspondingly, the quality of the admitted class remains much stronger under decisions made by the rubric embedding models. Notably the empirical results of our random forest models do not track the theoretical lines the same way ridge or LASSO models do, though this behavior is expected as random forests are not perfectly linear in their labels. 

Fig.~\ref{fig:app_robustness_rmse} reproduces Fig.~\ref{fig:rmse} for LASSO (top row) and random forest models (bottom row), respectively. At low levels of injected bias, rubric-based models remain competitive in predicting the proxy label. As bias increases, black-box and kitchen-sink models fit the proxy label more closely, while rubric-based models maintain lower RMSE against the original admissions scores. 

LASSO models were fit with the same computational setup as the ridge models, though we allocated 16 cores instead of 4, decreasing compute time to 3 hours. 

Random forests were fit on a CPU cluster using R 4.4 and the \texttt{ranger} R package~\citep{ranger_package}. Per tuning task, we allocated 4 CPU cores, 32 GB memory, and up to 12 hours of wall time on a remote compute cluster. However, the trees fit without exceeding memory usage of 2 GB, and all jobs completed within less than 4 hours. We fit 500-tree forests with 10 outer folds and 10-fold inner calibration. The main grid comprised 720 tasks (3 feature sets x 2 document variants x 6 bias values x 20 random iterations). We selected hyperparameters separately for each feature-set/document-variant/bias specification by five-fold cross-validation over 18 candidate configurations (\texttt{mtry} of 10\%, 33\%, or 67\% of features; \texttt{min.node.size} of 1, 5, or 10; and \texttt{sample.fraction} of 0.632 or 0.8.) Hyperparameters were held fixed across the 20 simulation iterations. The AUC analysis comprised 120 analogous tasks. Based on the observed task durations, we estimate that the complete RF suite required on the order of 2,000--3,000 CPU-hours, however this is an imputed estimate rather than a scheduler-accounting total.

\begin{figure}[p]
  \begin{center}
    \includegraphics[]{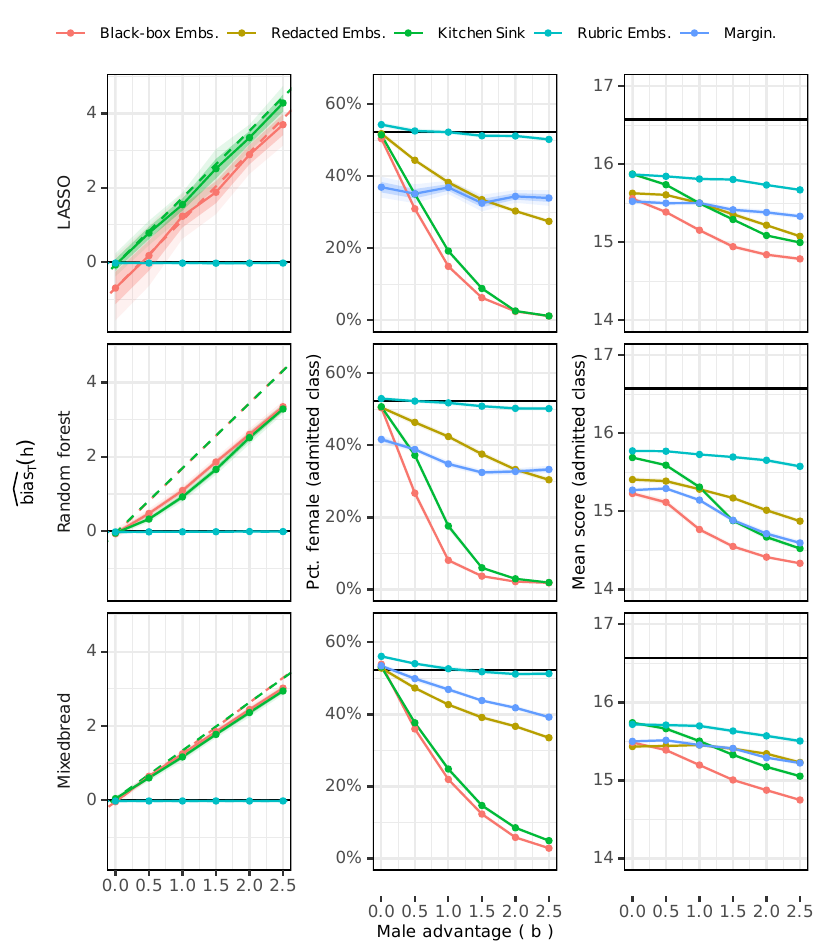}
  \end{center}
  \caption{\emph{%
    Replication of Fig.~\ref{fig:audit_all} for LASSO models (top row), random forests (middle row), and ridge models fit on embeddings from \emph{\texttt{mxbai-embed-large-v1}}~\citep{emb2024mxbai} (bottom row). \emph{Left:} Estimated bias \(\widehat{\D{bias}}_T(h)\) of fitted models \(h(x)\). Dashed line shows bias estimated according to Eq.~\eqref{eq:bias_decomp}. \emph{Middle:} Percentage of women in the admitted class under a top-20\% policy. \emph{Right:} Mean true score \(Y\) of the admitted class.
    Shading shows $\pm1$ and $\pm2$ standard errors, respectively. 
  }}
\label{fig:app_robustness_models}
\end{figure}

\begin{figure}[]
  \begin{center}
    \includegraphics{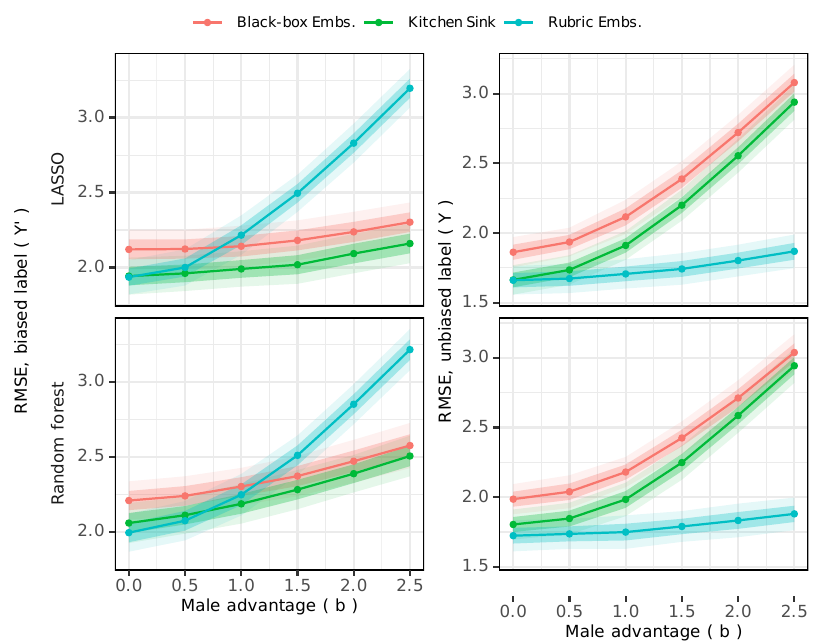}
  \end{center}
  \caption{\emph{%
    Replication of Fig.~\ref{fig:rmse} for  LASSO models (top row) and random forests (bottom row). \emph{Left:} RMSEs of rubric embedding, black-box embedding, and kitchen sink models, evaluated relative to proxy labels \(Y'\). \emph{Right:} RMSEs of rubric embedding, black-box embedding, and kitchen sink models, evaluated relative to ground truth \(Y\). Shading shows $\pm1$ and $\pm2$ standard errors, respectively. At \(b=0\), proxy labels retain zero-mean, unit-variance noise, so RMSEs against \(Y'\) and \(Y\) need not coincide.
  }}
\label{fig:app_robustness_rmse}
\end{figure}

\paragraph{Other embedding models}
Throughout our analysis, the general-purpose embeddings we used were from OpenAI's \texttt{text-embedding-3-large} model~\citep{openai2023textembedding}. To confirm our results do not depend on this choice, we replicate our main text results using the open source \texttt{mxbai-embed-large-v1} embedding model~\citep{emb2024mxbai}, generated locally using the \texttt{transformers} package~\citep{wolf-etal-2020-transformers}. Similar to \texttt{text-embedding-3-large}, we find models fit on \texttt{mxbai-embed-large-v1} achieve a high AUC in predicting gender (\(98.2\% \pm 0.7\%\)), and that \(\B E[\pi(T(X, m)) - \pi(T(X, f))]\) (the extent to which our representations encode gender), is \(0.67 \pm .02\). This corresponds to the smaller slope of the dashed lines we see in Fig.~\ref{fig:app_robustness_models}, bottom row, as our theoretical results predict. Despite the embeddings' slightly weaker encoding of gender, our audit bias and class composition results persist. 

\paragraph{Alternative outcome}

Our primary analysis treats observed admissions scores as a fixed reference outcome, although these scores may themselves contain bias. We therefore repeat the analysis using first-year core quantitative GPA among students who matriculated in 2024, an outcome much less likely to be subject to bias. We construct representations from their application materials and inject bias by inflating male and deflating female GPAs, varying $b$ from 0 to 1 (and noise standard deviation of 0.5). Within this matriculant sample, we select the top 20\% by predicted GPA. Rubric-based models select students with higher observed GPAs than black-box models and show substantially less change in selected-cohort gender composition as the injected advantage increases (Fig.~\ref{fig:app_robustness_outcomes}, bottom row). We do not estimate audit bias because we did not generate gender-transformed materials for this cohort.

\begin{figure}
  \begin{center}
    \includegraphics[]{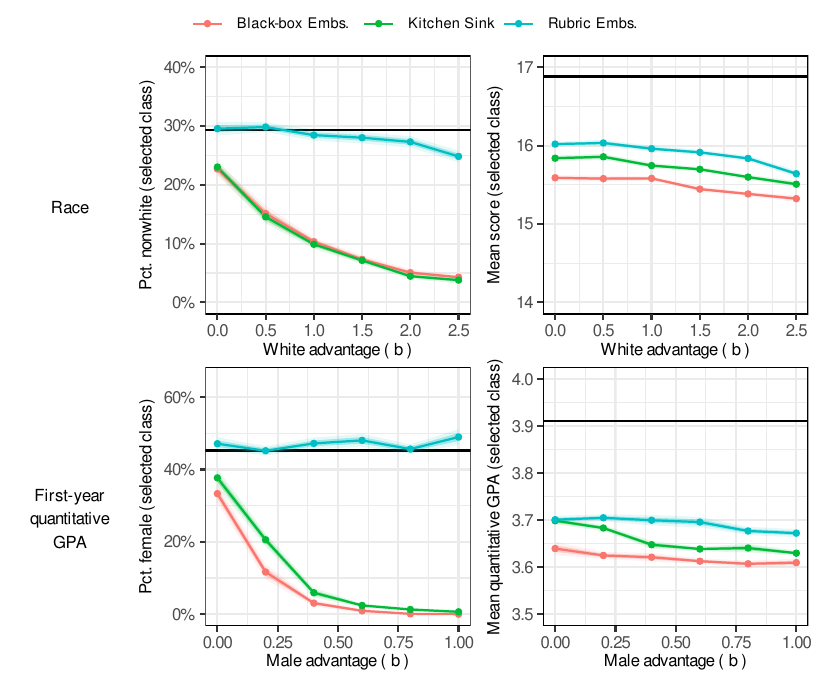}
  \end{center}
  \caption{\emph{%
    Replication of Fig.~\ref{fig:audit_all}, omitting the audit bias panel. The bottom row shows our results when advantage is applied to observed outcomes once admitted (GPA). The top row displays our results when advantage is given to white applicants (and non-whites are disadvantaged). \emph{Left:} Percentage of non-whites (top) or women (bottom) in the admitted class under a top-20\% policy. \emph{Right:} Mean true score \(Y\) of the admitted class (top), mean quantitative GPA of the admitted class (bottom).%
  }}
\label{fig:app_robustness_outcomes}
\end{figure}

\paragraph{Fewer rubric features}

To assess if the success of rubric embeddings is dependent on certain features we fit models on random subsets of \(k\) of the 811 total features included in our rubric embedding models (396 rubric features, structured covariates, and missingness indicators), from \(k = 81\) (10\%) to the full set. Bias remains near zero for every \(k\), but the quality of the admitted class increases with \(k\) (Fig.~\ref{fig:ksweep}, top row). We find a similar pattern for GPA and GRE scores (bottom row). Bias stays low even with few rubric items, though models with the most features admit the highest scoring class. 

\begin{figure}[p]
  \centering
  \includegraphics{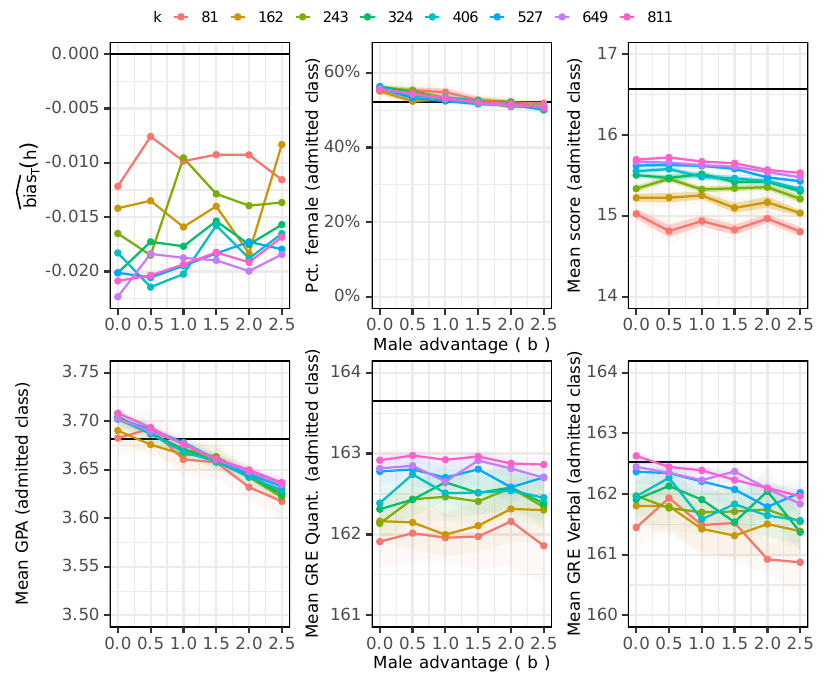}
  \caption{\emph{Rubric embedding models fit on random subsets of \(k\) features. Panels as in Fig.~\ref{fig:audit_all} and Fig.~\ref{fig:audit_covariates}. Shading shows $\pm1$ and $\pm2$ standard errors, respectively. Error bars are omitted from the top left panel for visual clarity, all overlap substantially.}}
  \label{fig:ksweep}
\end{figure}

\subsection{Forms of bias}

To explore whether our results are dependent on our choice of simulated bias, we reproduce our results two ways: we add an explicit bias on race instead of gender (an alternative form of disparate treatment) and a bias towards features that are correlated with being male (to test disparate impact).

\paragraph{Race} To assess if our results replicate under a bias for race, as opposed to gender, we create biased labels $Y_i'$ in an identical process to that presented in~\eqref{eq:biased_outcome}, however instead where \(G_i \in \{w,nw\}\) denotes whether the applicant is white or non-white and \(b\) gives the size of the advantage for white applicants. We did not generate the counterfactual materials \(T(x, w)\) and \(T(x, nw)\) representing the ``white-presenting'' and ``non-white-presenting'' versions of an applicant's materials. Without race-based audit study materials, it is not possible to carry out our empirical audit study analysis. Additionally, this analysis is restricted to domestic applicants, as an overwhelming majority of international applicants are missing race information. Regardless, as in the case of gender-biased labels, we see that predictive models trained on ``race-biased'' labels using black-box embeddings replicate the label bias in their predictions, whereas models trained using rubric embeddings fail to replicate the bias. Further, the class-composition results presented in the top row of Fig.~\ref{fig:app_robustness_outcomes} suggest the rubric-embedding based models select cohorts with higher original admissions scores and substantially less change in racial composition than black-box or kitchen-sink models. Consistent with this pattern, we find race is also encoded in general purpose embeddings: models trained to predict race using general purpose embeddings achieve an AUC of \(91.9\% \pm 3.0\%\).

\paragraph{Variables correlated with gender} We also inject an advantage for business, economics, and engineering majors, who disproportionately identify as male in our sample, and penalize other majors. Fig.~\ref{fig:app_audit_all_correlated} reproduces our main results under these biased labels. To avoid shrinking the coefficients on the rewarded major indicators, we exempt these three features from regularization in models that include them. Rubric-based models retain near-zero estimated audit bias and select a greater proportion of women than black-box models. However, their selected cohorts' original admissions scores decline as the advantage increases, eventually falling below those selected by black-box models. Because the rewarded major indicators are directly represented in the rubric, restricting the representation does not prevent the model from learning distorted weights on these retained features. This experiment illustrates a limitation of rubric embeddings: low audit bias does not ensure protection against all forms of
label bias. 

\begin{figure}[p]
  \begin{center}
    \includegraphics[]{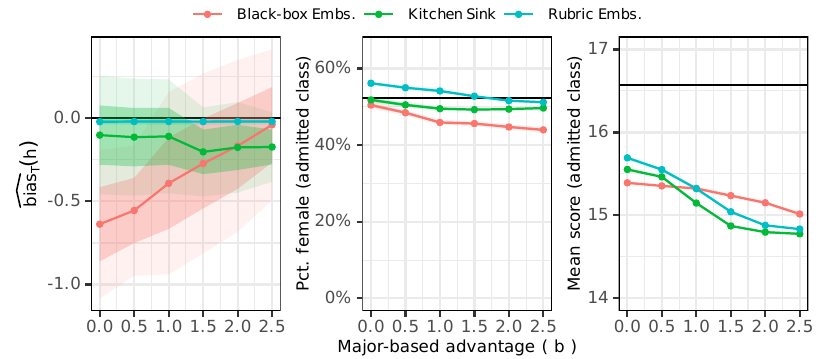}
  \end{center}
  \caption{\emph{%
    Replication of Fig.~\ref{fig:audit_all} with an advantage for business, economics, and engineering majors, rather than for men.  
  }}
\label{fig:app_audit_all_correlated}
\end{figure}

\clearpage

\begin{tcolorbox}[promptstyle, title=Audit Study Templating Prompt]
\begin{Verbatim}[breaklines, breakanywhere, fontsize=\small, commandchars=\~\<\>]
You are assisting in a correspondence (audit) experiment on admissions
to a public policy master's program. Reproduce the input document
character-for-character, with each gendered signal and each piece of
the applicant's identifying information wrapped in an inline marker
so downstream code can render gender-swapped variants.

OUTPUT: only the templated document. No preamble, no commentary, no
JSON wrapper, no code fences. Begin your output with the very first
character of the input document (or with a marker, if the document
begins with a name).

MARKER FORMAT

Markers are delimited by ASCII control characters:

  STX = \x02 (Start of Text)
  ETX = \x03 (End of Text)
  US  = \x1F (Unit Separator)

Gender marker — exactly four fields, separated by US:

  \x02original\x1Fneutral\x1Fmale\x1Ffemale\x03

The four fields are: the source span verbatim, the gender-neutral
variant, the male variant, and the female variant. If a particular
variant truly has no reasonable analogue (e.g. "her pregnancy" has
no clean male variant), use the literal token  null  in that field
(no quotes). Use null sparingly — almost everything has an analogue.

Name marker — exactly three fields, the first being the literal
string "NAME":

  \x02NAME\x1Fpart\x1Foriginal\x03

`part` MUST be EXACTLY one of:
  first | last | nickname | other | email | handle | url

The third field is the exact name token from the source.

INTEGRITY RULE — IMPORTANT

If you remove every marker and replace it with its `original` field
(field 1 of gender markers; field 3 of name markers), the result MUST
equal the input document byte-for-byte. Do NOT fix typos. Do NOT
normalize whitespace. Do NOT change anything outside a marker. The
input may contain markdown, HTML comments, page-break artifacts —
preserve all of it.

WHAT TO WRAP IN A GENDER MARKER

Every occurrence of an explicit gender signal that refers to the
APPLICANT (not third parties):
  - Pronouns: he/she, him/her, his/her, hers, himself/herself.
    Use "themself"/"themselves" for the neutral reflexive.
  - Gendered titles: Mr., Ms., Mrs., Miss
  - Gendered nouns: actress, chairman, businesswoman
  - Family roles applied to the applicant: father, daughter
  - Explicit gender statements: "as a woman", "being male"
  - Gendered modifiers: women's varsity team
  - Membership in gender-specific groups (sororities/fraternities,
    Boy/Girl Scouts, single-sex schools). Use a reasonable
    opposite-gender analogue (Eagle Scout ~promptarrow Girl Scout Gold Award;
    Alpha Phi sorority ~promptarrow a plausible fraternity).

WHAT NOT TO WRAP IN A GENDER MARKER

  - Names, emails, URLs, handles — wrap with NAME markers instead
  - Profession/interest correlates: nursing, engineering, ballet
  - Non-gendered honorifics: Dr., Prof., Rev.
  - Non-gendered pronouns: my, I, our
  - Gender signals referring to a third party (recommender, family,
    colleague). Leave them entirely outside any marker.

WHAT TO WRAP IN A NAME MARKER

Every occurrence of the APPLICANT'S name or an identifier that leaks
the name. Wrap each occurrence individually. When a third party
shares a name token (e.g. recommender shares the applicant's first
name), wrap ONLY the applicant's occurrences and leave the third
party's bare.

EXAMPLE

Input:
  "Jane Smith excelled. She wrote her thesis on social welfare."

Output:
  "\x02NAME\x1Ffirst\x1FJane\x03 \x02NAME\x1Flast\x1FSmith\x03 excelled. \x02She\x1FThey\x1FHe\x1FShe\x03 wrote \x02her\x1Ftheir\x1Fhis\x1Fher\x03 thesis on social welfare."
\end{Verbatim}
\end{tcolorbox}
\captionof{prompt}{\emph{The templating prompt, which instructs the model to wrap all applicant
  gender signals and name tokens in structured inline markers.}}
\label{prompt:template}

\begin{figure}
\vspace{2in}
\begin{tcolorbox}[promptstyle, unbreakable, title=Audit Study Validation Prompt]
\begin{Verbatim}[breaklines, fontsize=\small]
YOU ARE NOW REVIEWING the templated output that follows. The original
model produced this output by applying the rules above to a source
document. Your job is to find any APPLICANT gender signal that was
left outside a marker — i.e., anywhere the original model failed to
follow rule 2 above.

OUTPUT a JSON object with a single field:


{"missed": ["...", "...", ...]}

Each entry should be a short string identifying the missed signal and
its surrounding context, e.g. "her in 'complete her thesis'".

Return an EMPTY list if nothing was missed.

DO NOT flag:

  - Gender signals referring to third parties (recommender, family
    members, named non-applicants, colleagues, etc.). These are
    correctly left outside markers.
  - Non-gendered honorifics (Dr., Prof., Rev.)
  - Profession/interest correlates (nursing, ballet, engineering, etc.)
  - Generic references that are not tied to the applicant

Be strict but not paranoid: only flag tokens you are confident the
applicant of this document possesses or is described by.
\end{Verbatim}
\end{tcolorbox}
\captionsetup{type=prompt, hypcap=false}
\caption{\emph{The validation prompt, which instructs the model to confirm if the prior templating step missed any indicators of gender or applicant name.}}
\label{prompt:validation}
\end{figure}

\begin{figure}
\begin{tcolorbox}[promptstyle, unbreakable, title=Audit Study Repair Prompt]
\begin{Verbatim}[breaklines, fontsize=\small]
This template was flagged as potentially still containing incomplete
processing. The issues identified:

{issues_block}

Some of these flags may be false positives — use your judgment:

- If a flagged token correctly refers to a third party (recommender,
  family member, colleague, etc.), leave the existing template
  unchanged for that token.
- If a flagged token genuinely refers to the applicant and should
  have been wrapped, add the appropriate marker.
- For parse errors, fix the malformed marker(s) following the format
  rules in the system prompt.
- For integrity errors, restore the missing/modified text outside any
  marker so the marker-stripped reconstruction matches the source.

Output the full corrected template, no commentary.

PREVIOUS TEMPLATE:

{template}
\end{Verbatim}
\end{tcolorbox}
\captionsetup{type=prompt, hypcap=false}
\caption{\emph{The repair prompt, which instructs the model to repair any of the issues raised in the validation step.}}
\label{prompt:repair}
\end{figure}

%%%%%%%%%%%%%%%%%%%%%%%%%%%%%%%%%%%%%%%%%%%%%%%%%%%%%%%%%%%%

\end{document}